\documentclass[10pt, conference, compsocconf]{IEEEtran}

\usepackage{amsthm}
\usepackage{graphicx}
\usepackage{algorithm}
\usepackage{algorithmic}
\usepackage{amsmath}
\usepackage{multirow}
\ifCLASSINFOpdf
\else
\fi
\begin{document}
%
\title{Non-parametric Power-law Data Clustering}


\author{%
{Xuhui Fan{\small $^{1}$}, Yiling Zeng{\small $^{2}$}, Longbing Cao{\small $^{3}$} }%
\vspace{1.6mm}\\
\fontsize{10}{10}\selectfont\itshape
Advanced Analytics Institute\\
University of Technology Sydney, Australia\\
\fontsize{9}{9}\selectfont\ttfamily\upshape
$~^{1}$Xuhui.Fan@student.uts.edu.au\\
$~^{2,3}$\{Yiling.Zeng,LongBing.Cao\}@uts.edu.au%
}


%


\newtheorem{prop}{Proposition}
\newtheorem{mydef}{Definition}
\newtheorem{theorem}{Theorem}
\renewcommand{\algorithmicrequire}{\textbf{Input:}}
\renewcommand{\algorithmicensure}{\textbf{Output:}}
\maketitle

\begin{abstract}
It has always been a great challenge for clustering algorithms to  automatically determine the cluster numbers according to the distribution of datasets. Several approaches have been proposed to address this issue, including the recent promising work which incorporate Bayesian Nonparametrics into the $k$-means clustering procedure. This approach shows simplicity in implementation and solidity in theory, while it also provides a feasible way to inference in large scale datasets. However, several problems remains unsolved in this pioneering work, including the power-law data applicability, mechanism to merge centers to avoid the over-fitting problem, clustering order problem, e.t.c..
To address these issues, the Pitman-Yor Process based k-means (namely \emph{pyp-means}) is proposed in this paper. Taking advantage of the Pitman-Yor Process, \emph{pyp-means} treats clusters differently by dynamically and adaptively changing the threshold to guarantee the generation of power-law clustering results. Also, one center agglomeration procedure is integrated into the implementation to be able to
merge small but close clusters and then adaptively determine the cluster number. With more discussion on the clustering order, the convergence proof, complexity analysis and extension to spectral clustering, our approach is compared with  traditional clustering algorithm and variational inference methods. The advantages and properties of pyp-means are validated by experiments on both synthetic datasets and real world datasets.

\end{abstract}

\begin{IEEEkeywords}
Bayesian Non-parametrics; Pitman-Yor Process; power-law data structure; $k$-means clustering.
\end{IEEEkeywords}

%
\IEEEpeerreviewmaketitle

\section{Introduction}

Power-law  data is ubiquitous in real world. Examples include social networks on Facebook, topics in  web forums and citations among published papers. This kind of data differs from the traditional ones by the large fluctuations that occur in the tails of the distributions. Increased attentions have been received in recent years for detection the power-law phenomena and characterization the structure of such kind of data. Clustering is one essential technique for data structure learning owing to its capability of grouping data collections automatically. However, the key challenge of employing clustering on power-law data lies on the difficulties of inferring the cluster number, as well as determining the cluster sizes.

In the previous decades, various clustering methods have been proposed in dealing with different kinds of data. However, most of them, including classic k-means \cite{lloyd1982least}\cite{hartigan1979algorithm}\cite{bezdek1980convergence}, Mixture Models \cite{bishop2006pattern}\cite{figueiredo2002unsupervised}, Spectral Clustering\cite{shi2000normalized}\cite{von2007tutorial}, Mean Shift\cite{cheng1995mean}\cite{comaniciu2002mean} , etc., assume the cluster number to be a kind of prior information which should be provided by users, the value of which is usually unknown for the user. A few initial approaches\cite{bischof1999mdl}\cite{fraley1998many}\cite{Hamerly03learningthe}\cite{sugar2003finding} have been proposed to handle this unknown cluster number problem. However, most of them address this problem from the model selection criteria, and this leads to a dilemma in the selection of criteria.

Bayesian non-parametric learning, a fast growing research topic in recent years, can be utilized as an effective approach to address the parameter selection issue. Its core idea is to treat the required parameters, e.g., the cluster number, under a hyper-distribution and employ inference methods to learn the posterior probability of the latent variables given the observations. It demonstrates its significance contributions in parameter inference. However, it often suffers from the difficulty of designing learning schemes based on conjugate assumption, as well as computational complexity induced by inference and sampling . To address this problem, a new method called "DP-means"[14][15] has been proposed to bridge the classic k-means clustering and the non-parametric Dirichlet Process Gaussian Mixture Model (DPGMM). Taking advantage of the asymptotic zero-covariance property of Gaussian Mixture Models, \emph{dp-means} naturally introduces a fixed threshold to determine whether a data point should belong to an existing cluster or a new cluster should be created for it. It provides a unified view to combine Bayesian non-parametric methods and the hard clustering algorithms to address scale learning in large datasets.


However, several issues remain unsolved in this promising pioneering work. (i) The method is not designed for power-law data. A global threshold for all clusters may result in clusters with similar sizes. (ii) Mechanism to merge closed centers is needed. The algorithm may result in many small clusters. Some of them should be merged if they are closed enough. (iii) The clustering order influences the result. Strategies should be discussed more to address this issue.




This paper proposes a novel clustering approach, Pitman-Yor Process-means (\emph{pyp-means}), for clustering power-law data. A modified Pitman-Yor Process [16][17] is first proposed to approximate the power-law data structure in hard partition. Unlike the fixed threshold proposed in \emph{dp-means}, the modified Pitman-Yor Process introduces a deregulated threshold whose value changes in accordance to the cluster number during clustering. The larger the cluster number, the smaller the threshold would be. In this way, \emph{pyp-means} establishes a proper connection between the cluster number and threshold setting.

A center agglomeration procedure is also proposed to adaptively determine the cluster number. To address the issue that the clustering procedure may result in isolated small clusters in which are are not far from each other, we check the inter-distance between each pair of cluster centers and combine them if the distance of the two clusters are smaller than a value. This will prevent the clustering result from over fitting the power-law distribution while taking no account of the real data distribution.

The heuristic "furthest first" strategy is more discussed here to address the data order issue. We further prove that once the cluster number stops increasing, arbitrary order of the remained data points will result in the same clustering result.

The convergence of pyp-means is proved and the complexity of the algorithm is  analyzed. We further extend our method to spectral clustering to prove the effectiveness of our work.

The contribution of our work is summarised  as follows:




\begin{itemize}
\item we extend the newly proposed dp-means to the modified Pitman-Yor process based k-means algorithm to address the power-law data, which is a generalization and being able to cluster both the power-law dataset and normal dataset.
\item we integrate a center agglomeration procedure into the main implementation to overcome the overfitting problem.
\item we introduce a heuristic "furthest first" strategy to address the data order issue during clustering procedure.
\item we prove the convergence of \emph{pyp-means} and calculate the complexity of the algorithm. We also extend the method to fit spectral clustering to expand our approach to multiple clustering algorithms.
\end{itemize}

The remaining part of the paper is organized as follows. Section \ref{sec_2} introduces related work on Gaussian Mixture Models with its new derivatives and the Pitman-Yor Process. The modified Pitman-Yor Process towards power-law data is represented in Section \ref{sec_3}. Then followed by Section \ref{sec_4} detail discusses our proposed approach \emph{pyp-means}, including the main implementation, the strategy for data order issue, and the center agglomeration procedure to avoid overfitting. Further discussion on pyp-means' convergence proof, complexity analysis and its extension to spectral clustering can be found in Section \ref{sec_5}. Section \ref{sec_7} introduces experimental results on different datasets to prove the effectiveness of our work. Followed by the last section which draws a conclusion of this paper.

\section{Background} \label{sec_2}
We briefly introduce the relations between Gaussian Mixture Models with its derivatives to $k$-means clsutering\cite{bishop2006pattern},  and the pitman-yor process.

\subsection{Gaussian Mixture Models with its new derivatives to $k$-means}
Gaussian Mixture Models(GMM) treats a dataset as a set containing the samples from several Gaussian distributions. The likelihood of a data point $\boldsymbol{x}$ in the dataset can be calculated as::
\begin{equation}
p(\boldsymbol{x})=\sum_{k=1}^c \pi_k \mathcal{N}(\boldsymbol{x}|\boldsymbol{\mu}_k, \Sigma_k)
\end{equation}
Here $c$ is the components' number, $\pi_k$ denotes the proportion of component $k$, and $\mathcal{N}(\boldsymbol{x}|\boldsymbol{\mu}_k, \Sigma_k)$ is $\boldsymbol{x}$'s Gaussian likelihood in component $k$.

The local maximum of a Gaussian Mixture Model can be achieved by applying Expectation Maximization with iteration of the following equations:
\begin{equation}
\begin{split}
& \boldsymbol{\mu}_k=\frac{1}{N_k}\sum_{n=1}^N \gamma(z_{nk})\boldsymbol{x}_n\\
& N_k = \sum_{n=1}^N \gamma(z_{nk}) \\
& \Sigma_k = \frac{1}{N_k} \sum_{n=1}^N \gamma(z_{nk})(\boldsymbol{x}_n-\boldsymbol{\mu}_k)(\boldsymbol{x}_n-\boldsymbol{\mu}_k)^T
\end{split}
\end{equation}
Here $\gamma(z_{nk})=\frac{\pi_k \mathcal{N}(\boldsymbol{x}|\boldsymbol{\mu}_k, \Sigma_k)}{\sum_{k=1}^c \pi_k \mathcal{N}(\boldsymbol{x}|\boldsymbol{\mu}_k, \Sigma_k)}$ denotes the probability of assigning data point $\boldsymbol{x}_n$ to cluster $k$. As a result, it is regarded as a kind of soft clustering \cite{nock2006weighting} which is different from traditional hard clustering (e.g. k-means) in the way of assigning single point to multiple clusters with probabilities.

Actually, the ``asymptotic'' link (i.e. zero-variance limit) between the GMM and $k$-means clustering is a well-known result as in \cite{bishop2006pattern}\cite{roweis1999unifying}. More specifically, the covariance matrices of all mixture component in GMM are assumed to be $\epsilon I_{d\times d}$, and $p(\boldsymbol{x})$ becomes
\begin{equation}
p(\boldsymbol{x}|\boldsymbol{\mu}_k, \Sigma_k)=\frac{1}{(2\pi\epsilon)^{d/2}}\exp{(-\frac{1}{2\epsilon}\|\boldsymbol{x}-\boldsymbol{\mu}_k\|^2)}
\end{equation}
with $\gamma(z_{nk})$'s calculation is changed as:
\begin{equation} \label{eq_5}
\gamma(z_{nk})=\frac{\pi_k\exp\{-\|\boldsymbol{x}_n-\boldsymbol{\mu}_k\|^2/2\epsilon\}}{\sum_j\pi_j\exp\{-\|\boldsymbol{x}_n-\boldsymbol{\mu}_j\|^2/2\epsilon\}}
\end{equation}
Consider the case $\epsilon\to 0$, the smallest term of $\{\|\boldsymbol{x}_n-\boldsymbol{\mu}_j\|^2\}_{j=1}^c$ will dominate  the denominator of Eq. (\ref{eq_5}). Thus, $\gamma(z_{n,k})$ becomes:
\begin{equation}
\mathcal{\gamma}(z_{n,k}) = \left\{ \begin{array}{ll}
1 &  k = \arg\min_j\{\|\boldsymbol{x}_n-\boldsymbol{\mu}_j\|^2\}_{j=1}^c\\
0 &  $otherwise$
\end{array} \right.
\end{equation}


In this case, GMM degenerates into $k$-means which assigns each points to its nearest clustering with probability of 1.

Both GMM and k-means are suffering from the selection of cluster numbers. To address the problem, Dirichlet Process \cite{jordanrevisiting}\cite{kulis2011revisiting} is introduced into k-means recently. Taking advantage of the Dirichlet Process, a distance threshold can be generated to prevent data points from being assigned to a cluster if the distances exceed the threshold. If a data point fails to be assigned to all clusters, a new cluster will be created by taking it as the cluster center.

More specifically, a dirichlet process can be denoted as $\mathcal{DP}(\alpha, H)$, with the hyper-parameter $\alpha$ and basement distribution $H$. The hyper-parameter alfa can be further written in the form of $\alpha=\exp{(-\frac{\lambda}{2\epsilon})}$ for some $\lambda$, and the base measurement $H$ is used to generate Gaussian distribution $\mathcal{N}(0, \rho I)$. Gibbs sampling is take to address the the above process. The probability used in Gibbs sampling can be written as

\begin{equation}
\mathcal{\gamma}_{z_{nk}} = \left\{ \begin{array}{ll}
n_{-i,k}\cdot\exp{(-\frac{1}{2\epsilon}\|\boldsymbol{x}_i-\boldsymbol{\mu}_k\|^2)}/Z &  k$-th$ $ cluster$\\
\exp{(-\frac{1}{2\epsilon}\lambda-\frac{1}{2(\epsilon+\rho)}\|\boldsymbol{x}_i\|^2)}/Z &  $new cluster$
\end{array} \right.
\end{equation}
Where $k(1\le k\le c)$ denotes the existing $k$-th cluster, $n_{-i, k}$ represents the size of the $k$-th cluster excluding data point $\boldsymbol{x}_i$ and $Z$ is the normalizing constant.

While $\epsilon\to0$, the allocated label for data point $\boldsymbol{x}_n$ becomes:
\begin{equation}
l_n = \left\{ \begin{array}{ll}
k_o &  k_0 = \arg\min_j\{\lambda, \|\boldsymbol{x}_n-\boldsymbol{\mu}_j\|^2\}_{j=1}^c\\
c+1 &  \lambda=\arg\min_j\{\lambda, \|\boldsymbol{x}_n-\boldsymbol{\mu}_j\|^2\}_{j=1}^c
\end{array} \right.
\end{equation}
Once Gibbs sampling assigns a new label $c+1$ to $\boldsymbol{x}_n$, a new cluster will be generated with Gaussian distribution $\mathcal{N}(0, \rho I)$. Within finite steps, local minimum can be achieved. All data points are assigned to corresponding clusters with distances to the centers smaller than the given threshold.


Taking advantage of Dirichlet Process, this work, known as \emph{dp-means}, successfully combines the prior information (hyperparameter $\alpha=\exp{(-\frac{\lambda}{2\epsilon})}$) and local information of each component (the gaussian distribution $\mathcal{N}(\boldsymbol{\mu}, \epsilon I)$).

\emph{Dp-means} treat each component equally. One universal threshold is set for all clusters. In this way, clustering result tends to contain cluster with similar sizes. However, clusters in real world dataset usually vary a lot. They typically obey power-law distributions. Therefore, it would be convenient if the method is able to generate more reasonable clustering results which satisfy power-law distribution.

Though a wonderful work, a systematic learning on the unsolved problems, including an increasing cluster number problem, the clustering order problem, complexity analysis, e.t.c., should be focused. Other practical issues including power-law data approximating, parameter $\lambda$'s adjustment also needs to be further investigated.


\subsection{Pitman-Yor Process}

 Pitman-Yor process (\emph{py-process})\cite{pitman1997two}\cite{ishwaran2001gibbs} is a generalization of Dirichlet Process. In \emph{py-process}, a discount parameter $d$ is added to increase the probability of new class generation. Due to the discount parameter $d$'s tuning effect, it becomes a suitable model to depict power-law data. \emph{py process} degenerate to classic dirichlet process when $d$ is set to 0.



A P$\acute{o}$lya urn scheme is used here to explain the py-process' generative paradigm in the technical perspective. In this scheme, objects of interests are represented as colored balls contained in an urn. At the beginning, the urn is empty. All balls are uncolored. We pick the first ball, paint it with a certain color and put it into the urn. In the following steps, we pick one ball each time, color it and put it into the urn.  The color of the ball is allocated according to the following probability.

\begin{equation}
\pi_{i,k} = \left\{ \begin{array}{ll}
\frac{n_{-i, k}-d}{\lambda+n_{-i}} & $existing $ k$-th$(1\le k\le c) $ color$\\
\frac{\lambda+c\cdot d}{\lambda+n_{-i}} &  k = c+1 $ new color$
\end{array} \right.
\end{equation}
Where $i$ denotes the $i$-th ball picked, $k$ denotes the color assigned to $i$, $n_{-i,k}$ denote the number of balls in color $k$ exclude ball $i$, $n_{-i}$ denotes the whole number of balls without ball $i$.

 The process continues until all balls are painted and put into the urn. While the size of each cluster is fixed, the joint probability is unchanged, which refers as ``exchangeability''.

\emph{Py-process} preserves Dirichlet Process' 'rich get richer' property during the process of assigning colors to balls. The larger size of balls in a certain color, the greater probability that the new ball will be painted in this color. Thanks to the discount parameter d, the probability of generating a new color in py-process is greater than that of DP. It can be easily proven that py-process draws colors to data points in a power-law scheme. Therefore, it would be promising if py-process in incorporated into clustering algorithms to help address the power-law data.

\section{Modified Pitman-yor Process for hard clustering} \label{sec_3}

Power-law data\cite{Clauset:2009:PDE:1655787.1655789}, also named as heavy-tailed behavior data, represents the case that the frequency or the size of some data cluster obey the
exponential distribution, i.e., more small sized subsets of cluster data are coming up. In reality life, a wide range of the data obeys power-law data, including the frequencies
of words in languages, the populations of cities, the intensities of earthquakes. Under most situations, these kind of findings in power-law data would be considered as noisy or
defective. However, these are at the same time some of the most interesting part from the whole observations.

In our scenarios, one cluster's data are denoted as the same, and we define the cluster size follows according to the power-law distribution. In contrast to the ordinary data, its clustering encounters more difficulties, such as the trivial cluster discovery, the cluster number determination and the related imbalanced problem.


A general form of the density function of power-law data is stated as:
\begin{equation}
p(x)\propto L(x) x^{-\alpha}
\end{equation}
where $L(x)$ is a slowly varying function, which is any function that satisfies $\lim_{x\to\infty}L(tx)/L(x)=1$ with $t$ constant,  and $\alpha(\alpha>1)$ is one decreasing parameter.

Regarding to these difficulties in power-law data clustering, traditional clustering methods tend to group the small size clusters into major clusters or simply treat them as
noisy data points. It is un-proper while these trivial clusters may still be important to the whole data structure. Many soft clustering methods including the \emph{py-process} have been put
forward to effectively mining this kind of data. They have received good results, however, most of them still suffer the complexity problem in implementation and high conditions
required. To the best of our knowledge, little work has been done on the hard clustering scenario, nor an equivalence connected with the classic $k$-means clustering.


With the core idea in \emph{py-process} of increasing the new cluster generation's probability, we revise the concentration parameter from $\lambda$ to $\lambda\cdot(\theta)^{\ln c}$
 in \emph{dp-means}.

More specifically, during each ball's color painting in the P$\acute{o}$lya urn scheme, the color allocated according to the paradigm below:
\begin{equation} \label{eq_4}
\pi_{i, k} = \left\{ \begin{array}{ll}
n_{-i, k}\cdot\exp{(\frac{1}{2\epsilon}\theta)}/Z & k$-th$(1\le k\le c) $ cluster$ \\
\lambda\cdot\exp{(\frac{\ln c}{2\epsilon}\theta)}/Z & $the new cluster$
\end{array} \right.
\end{equation}

It is quite straightforward to see that the balls are exchangeable, which is quite basic for the power-law approximating.
\begin{prop}
The revised allocation paradigm in Eq. (\ref{eq_4}) still keeps the exchangeability property, i.e., the joint probability of a data set is not affected by their orders given each cluster size fixed.
\end{prop}
\begin{proof}
Assume the cluster number is $c$, each cluster's number is $\{c_i\}_{i=1}^c$. Thus, the joint probability of the data points is
\begin{equation}
P(\boldsymbol{X})=\lambda^c\cdot\prod_{i=1}^c (c_i!\prod_{j=1}^{c_i} \exp{(-\frac{\ln j}{2\epsilon})})/{Z^n}
\end{equation}
The equation is determined only these variables, which is exchangeability preserved.
\end{proof}

In our revision, while $\theta$ is fixed at 1, then it is the normal \emph{dp-means}.





\section{py-process means} \label{sec_4}

Benefit from the revised color allocation paradigm (Eq. (\ref{eq_4})), we extend the existed \emph{dp-means} algorithm to do a further generalization, which named as ``\emph{rpy-means}''.
The induction strategy is also quite similar as \emph{dp-means}. Employing the same
setting on both of the finite and infinite Gaussian Mixture Model ($p(\boldsymbol{x})=\sum_{k} \pi_k \mathcal{N}(\boldsymbol{x}|\boldsymbol{\mu}_k, \epsilon I)$), the
parameters are modified as $\{\lambda = \exp{(-\frac{\lambda}{2\epsilon})}, \theta = \exp{(-\frac{\theta}{2\epsilon})}\}$ in our revised pitman-yor process approximating method.

This leads to the related probability of data point $i$ assigning to an existed cluster $k$ as:
\begin{equation}
p_{i, k}=\frac{n_{-i, k}\cdot\exp{(-\frac{1}{2\epsilon}\theta-\frac{1}{2\epsilon}\|\boldsymbol{x}_i-\boldsymbol{\mu}_k\|^2)}}{Z}
\end{equation}
following the probability to the new cluster as:
\begin{equation}
p_{i, new} = \frac{\exp{(-\frac{1}{2\epsilon}(\lambda-\ln c\cdot\theta)-\frac{1}{\epsilon+\rho}\|\boldsymbol{x}_i\|^2)}}{Z}
\end{equation}

While $\epsilon\to0$, the dominating term in $p_{i, k}, p_{i, new}$ is the minimal value of $\{\lambda-\ln c\cdot\theta, \{\theta+\|\boldsymbol{x}_i-\boldsymbol{\mu}_k\|^2\}_{k=1}^c\}$, leading to the cluster allocation paradigm as:
\begin{equation} \label{eq_1}
l_n = \left\{ \begin{array}{ll}
k_o &  \|\boldsymbol{x}_n-\boldsymbol{\mu}_{k_o}\|^2 = \arg\min\{\lambda-\ln c\cdot\theta, \|\boldsymbol{x}_n-\boldsymbol{\mu}_j\|^2\}_{j=1}^c\\
new &  \lambda-\ln c\cdot\theta= \arg\min\{\lambda-\ln c\cdot\theta, \|\boldsymbol{x}_n-\boldsymbol{\mu}_j\|^2\}_{j=1}^c
\end{array} \right.
\end{equation}
Here cluster number $c$ is constrained to $c<\exp{(\lambda/\theta)}$ to avoid the minimal distance problem.
By shorten the threshold value in accordance to the cluster number, more hidden clusters could be discovered and the clusters would also be more compact than a larger threshold.





\subsection{main implementation}
One stage of our main implementations in the \emph{pyp-means} clustering is quite similar as in \emph{dp-means}. However, differences come up in the fluctuated threshold
during the clustering procedure and an stepwise/adaptive density checking procedure.

One definition on the reach of cluster centers is first made to clarify the notation.
\begin{mydef} \label{de_1}
Any data point $\boldsymbol{x}$ that lies within the ball $b(\boldsymbol{\mu}_k, \lambda)$ is said that $\boldsymbol{x}$ is $\lambda$-in of $\boldsymbol{\mu}_k$ ($\lambda$-in
data), while being outside from the ball $b(\boldsymbol{\mu}_k, \lambda)$ is said to be $\lambda$-out of $\boldsymbol{\mu}_k$($\lambda$-out data).
\end{mydef}

Under Def. (\ref{de_1}), the \emph{dp-means} gets results that all of the data points are $\lambda$-in of centers $\{\boldsymbol{\mu}_k\}_{k=1}^c$.

Figure \ref{fig_2}. depicts the process of our implementation.
\begin{figure}[htbp]
\centering
\includegraphics[scale=0.40, width = 0.40 \textwidth, bb = 128 532 362 633, clip]{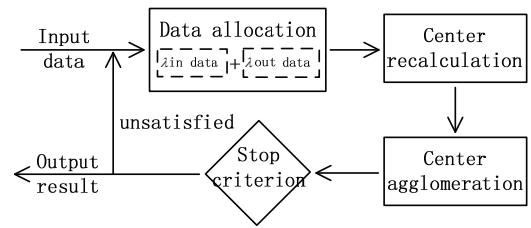}
\caption{Main implementation illustration}
\label{fig_2}
\end{figure}

The whole implementation consists of three procedures: data partition, center recalculation and center agglomeration. Data partition procedure shares
similarities with the existed \emph{dp-means}, which divides the data into $\lambda$-in data and $\lambda$-out data. For the $\lambda$-in data, its clustering method is according
to the usual way as $k$-means, while the $\lambda$-out data's clustering employs an adaptive way to determine the cluster, which would be detail discussed later. The center-recalculation
procedure is the same as the corresponding step in $k$-means. The center agglomeration procedure is one that to avoid too many trivial clusters. Details would be discussed later.

The detail implementation of our proposed \emph{pyp-means} is shown in Algorithm \ref{alg_1}.
\begin{algorithm}[htbp]
\caption{\emph{pyp-means}}
\label{alg_1}
\begin{algorithmic}
    \REQUIRE $\boldsymbol{x}_1,\cdots, \boldsymbol{x}_n$; $\lambda, \theta$, \emph{py-process}'s parameter
    \ENSURE clusters $l_1, \cdots, l_c$ and the number of clusters $c$
    \STATE Initialize $c=1$; initialize cluster center $\boldsymbol{\mu}_1$
    \REPEAT
    \FOR {each point $\{\boldsymbol{x}_i, i=1,\cdots, n\}$}
    \STATE compute $d_{ik}=\|\boldsymbol{x}_i-\boldsymbol{\mu}_k\|^2$ for $c=1,\ldots, c$.
    \IF {$\min_k d_{ik}-\theta > \lambda-\ln c\cdot\theta$}
    \STATE put $i$ into un-clustered set $D_r$
    \ELSE
    \STATE set $z_i=\arg\min_{k} d_{ik}$
    \ENDIF
    \ENDFOR
    \STATE re-clustering the remained un-clustered data set $D_r$
    \STATE employ the agglomeration procedure check
    \STATE update $c$, $\boldsymbol{\mu}_k$
    \UNTIL{converge}
\end{algorithmic}
\end{algorithm}


The objective function is identified as the cost of all inter-cluster distance ($km$-cost) adding one penalty term:
\begin{equation} \label{eq_2}
\begin{split}
\arg\min_{l_1,\ldots, l_n}&\quad\sum_{j=1}^c\sum_{i\in l_i}\|\boldsymbol{x}_i-\boldsymbol{\mu}_j\|^2+(\lambda-\ln c\cdot\theta)c\\
\textrm{where} & \quad \boldsymbol{\mu}_k=\frac{1}{|l_k|}\sum_{\boldsymbol{x}\in l_k}\boldsymbol{x}
\end{split}
\end{equation}
While $km$-cost tends to seek a larger $c$ value and the $c$-penalty term's value increases with $c$ value increases, the minimum value we are seeking is a trade-off in considering both of the cases.


\subsection{re-clustering on $D_r$}
The clustering order on the data will affect its performance in our \emph{pyp-means}, suffering from the same problem as in \cite{jordanrevisiting}. We discuss this problem here in accordance with the two stages of our main implementation: $\lambda$-in data clustering and re-clustering on $\lambda$-out data.




On clustering  $\lambda$-in data points, arbitrary order of these data points would result in the same clustering performance. This assertion applies for the classic $k$-means clustering.

On the contrary, the clustering order of $\lambda$-out data points would affect the center's determination in a sequence. The new generated different centers would be surely affect the data belonging. To explore this complex situation, a heuristic search method called ``furthest first'' is employed here. From the start of the re-clustering, we choose the data point $i_0$ whose shortest distance to all the existed centers are the largest, i.e., $i_0=\arg\max_i \{d_i|d_i=\min_k^c d_{ik}\}$ and set it as the new cluster center. Then we remove data point $i_0$ from $D_r$ and recursively do re-clustering.


One benefit of our ``furthest first'' is that we can avoid the generating of new clusters once the cluster number stopped increasing. Then arbitrary order of remained data points would not affect the clustering performance. This saves the computational cost in defining the centers.


\begin{prop}
During each iteration of our method, once no remaining $\lambda$-out data point becoming new cluster, the following data points would not be new centers either.
\end{prop}
\begin{proof}
Assume that the shortest distances of the remaining $D_r$ are $d_{(1)}<d_{(2)}<\cdots<d_{(r)}$, corresponding to the variables $\boldsymbol{x}_{(1)}, \cdots, \boldsymbol{x}_{(r)}$. While the $\boldsymbol{x}_{(k)}(1\le k\le r)$ does not come to be one new center, $d_{(k)}<\lambda-\ln c\cdot\theta$, then the threshold becomes fixed. Similar as the Apriori rules, then all the remaining data points $\{\boldsymbol{x}_{(l)}\}_{l>k}^r$ with $\{d_{(l)}<\lambda-\ln c\cdot\theta\}_{l>k}^r$ will belong to the existed clusters.

Thus, our selection order will not generate the new cluster centers.
\end{proof}
We formalize our re-clustering procedure on the remaining dataset $D_r$ as Algorithm \ref{alg_2}.
\begin{algorithm}[htbp]
\caption{re-clustering on $D_r$}
\label{alg_2}
\begin{algorithmic}
    \REQUIRE remain dataset $D_r$; $\lambda, \theta$, \emph{py-process}'s parameter; generated centers $\{\boldsymbol{\mu}_k\}_{k=1}^{c_0}$
    \ENSURE clusters $l_1, \cdots, l_c$ and the number of clusters $c$
    \FOR {each data point in $D_r$}
    \STATE compute $d_{ik}=\|\boldsymbol{x}_i-\boldsymbol{\mu}_k\|^2$ for $k=1,\ldots, c$,
    \STATE select the shortest one for each data point in $D_r$, denoted as $\{d_k\}_{k=1}^r$
    \STATE order $\{d_k\}_{k=1}^r$ from largest to smallest
    \IF {$d_1-\theta > \lambda-\ln c\cdot\theta$}
    \STATE c = c+1; set $\boldsymbol{x}_{(1)}$ as one new center $\boldsymbol{\mu}_{c+1}$
    \ELSE
    \STATE put $\boldsymbol{x}_{(1)}$ into the existed nearest cluster
    \ENDIF
    \ENDFOR
\end{algorithmic}
\end{algorithm}

\subsection{center agglomeration procedure} \label{sec_9}
In \emph{dp-means}, new cluster generated while new $\lambda$-out data encountered, however, it never disappeared even if it gets much closer to another cluster. This could result in an overfitting problem on dividing one dense clusters into two parts. On the other hand, this special overfitting results in an unproper smaller threshold on determining the valid cluster number. Thus, we need to adaptively determine the cluster number.



From the following proposition, we can evaluate the $km$-cost value's change while two clusters combine into one. Then the condition of clustering combining could be well established.

Assuming that $\lambda, \theta$ are the pre-defined parameter, c is the current cluster number, $\{\boldsymbol{\mu}_i\}_{i=1}^2$ and $\{n_i\}_{i=1}^2$ are the corresponding cluster center and cluster size, we have the following proposition:
\begin{prop}
If two clusters satisfy $\|\boldsymbol{\mu}_1-\boldsymbol{\mu}_2\|^2 < \frac{n_1+n_2}{n_1n_2} (\lambda-\theta\cdot\ln\frac{(c+1)^{(c+1)}}{c^c})$, then combining these two clusters could reduce the value of the objective function (Eq. (\ref{eq_2})).
\end{prop}
\begin{proof}
Assume $\{\boldsymbol{x}_i^{(1)}\}_{i=1}^{n_1}$ and $\{\boldsymbol{x}_i^{(2)}\}_{i=1}^{n_2}$ are the two closed cluster, with $\{\boldsymbol{y}_i\}_{i=1}^{n_1+n_2}$ denoting their combined clusters. Then the two cluster center satisfy the condition with the new combined cluster center $\boldsymbol{\mu}$:
\begin{equation}
n_1\boldsymbol{\mu}_1+n_2\boldsymbol{\mu}_2 = (n_1+n_2)\boldsymbol{\mu}
\end{equation}

We first show that the combination of two cluster would result an increase in the $km$-cost value,
\begin{equation}
\begin{split}
&\sum_{k=1}^{n_1+n_2} \|\boldsymbol{y}_k-\boldsymbol{\mu}\|^2-(\sum_{i=1}^{n_1} \|\boldsymbol{x}_i^{(1)}-\boldsymbol{\mu}_1\|^2+\sum_{i=1}^{n_2} \|\boldsymbol{x}_2^{(1)}-\boldsymbol{\mu}_2\|^2) \\
= &\sum_{k=1}^{n_1+n_2} \|\boldsymbol{y}_k\|^2-(n_1+n_2)\|\boldsymbol{\mu}\|^2\\
& -(\sum_{i=1}^{n_1} \|\boldsymbol{x}_i^{(1)}\|^2+\sum_{j=1}^{n_2} \|\boldsymbol{x}_j^{(2)}\|^2-n_1\|\boldsymbol{\mu}_1\|^2-n_2\|\boldsymbol{\mu}_2\|^2)\\
= &n_1\|\boldsymbol{\mu}_1\|^2+n_2\|\boldsymbol{\mu}_2\|^2-\frac{\|n_1\boldsymbol{\mu}_1+n_2\boldsymbol{\mu}_2\|^2}{n_1+n_2}\\
= & \frac{n_1n_2\|\boldsymbol{\mu}_1-\boldsymbol{\mu}_2\|^2}{n_1+n_2}\ge0
\end{split}
\end{equation}

Due to the cluster number reducement, the cluster number $c$-penalty term jumps from $(\lambda-\ln(c+1)\cdot\theta)(c+1)$ to $(\lambda-\ln c\cdot\theta)c$. Thus, if the condition satisfied
\begin{equation}
\begin{split}
& \frac{n_1n_2\|\boldsymbol{\mu}_1-\boldsymbol{\mu}_2\|^2}{n_1+n_2}\le (\lambda-\ln(c+1)\cdot\theta)(c+1)-(\lambda-\ln c\cdot\theta)c\\
&\Leftrightarrow   \|\boldsymbol{\mu}_1-\boldsymbol{\mu}_2\|^2 < \frac{n_1+n_2}{n_1n_2} (\lambda-\theta\cdot\ln\frac{(c+1)^{(c+1)}}{c^c})
\end{split}
\end{equation}
Our objective function is decreasing with the two clusters' combination.
\end{proof}

The following simple prototype illustrates our idea more clearly. $A, B, C, D$ are four data points to be clusterd. Assume the threshold is $r$ and we have previously clustered $A, B$ and $C, D$ being individual clusters. According to the agglomeration procedure, we combine these two clusters since the distance of the two centers are $0.8r$, satisfies the condition $0.8r \le\frac{2+2}{2\cdot2}r=r$. If we do not employ this procedure, then these two clusters would remain the same, leading to a unsatisfied result.
\begin{figure}[htbp]
\centering
\includegraphics[scale=0.40, width = 0.4 \textwidth, bb = 125 463 285 508, clip]{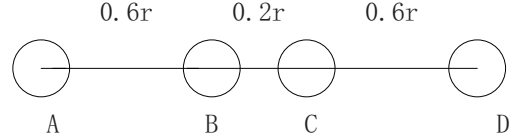}
\caption{Agglomeration procedure illustration}
\label{fig_5}
\end{figure}

In our detail implementation, each time after the cluster centers re-calculated, we run this agglomeration procedure. By checking if any pair of cluster centers satisfies the condition, we can effectively prevent the above situations.


\section{Further Discussion} \label{sec_5}
The work is extended here for further discussion, including the convergence analysis, the complex analysis, and possible extension to spectral clustering.

\subsection{convergence analysis}
Guaranteeing a local minimum value within finite steps is vital in our \emph{pyp-means}. We approach this goal by first showing that the objective function (Eq. (\ref{eq_2}))
strictly decreases during each iteration.

\begin{prop} \label{prop_2}
The objective function (Eq. (\ref{eq_2})) is to be strictly decreasing during each iteration we have applied in Algorithm \ref{alg_1}, until a local optimal point reached.
\end{prop}
\begin{proof}
The iteration is divided into three stages: partitioning data points; updating cluster centers; agglomeration procedure.

In Partitioning data points stage, the distance between $\lambda$-in data and its newly belonging cluster center would not increased, this is confirmed by \cite{selim1984k}.
The $\lambda$-out data are set as new cluster centers, this shrinkage the cost from to the $c$ penalty value $\lambda-c\cdot\theta$. What is more, as $c$ increases, the $c$
penalty value $\lambda-c\cdot\theta$ decreases, which reduces the objective function more.

In the updating cluster centers stage, the mean representation is always the optimal selection with the least cost value.



In agglomeration procedure, the objective function is strictly decreasing as provided in Section \ref{sec_9}.

Thus, the objective function is strictly decreasing.
\end{proof}

Employing the similar idea of \cite{selim1984k}, the convergence property of our \emph{pyp-means} could be easily obtained.
\begin{theorem} \label{theorem_1}
pyp-means converges to a partial optimal solution of the objective function (Eq. \ref{eq_2}) in a finite number of iterations.
\end{theorem}
\begin{proof}
As finite number of data points, we get finite partitions of data points in the maximum.


Assume our declaim is not true, which means that there exist $r_1\ne r_2$, such that $J_{r_1}=J_{r_2}$. Without loss of generality, we set $r_1 > r_2$.

According to Proposition \ref{prop_2}, our objective function $J_r$ strictly decrease while $r$ increases. Thus, for any $n_1 > n_2$, we get $J_{n_1}<J_{n_2}$. The inequality
also applies to $r_1, r_2$. This leads to the fact that $J_{r_1}<J_{r_2}$, which is contradict to our assumptions. Thus, the assumption does not success and we get our conclusion.
\end{proof}

Under Theorem \ref{theorem_1}, we can ensure the our procedure could reach a local minimum within finite steps.

\subsection{complexity analysis}

Our proposed \emph{pyp-means} is scalable to the number of data points $n$ and the final cluster number $c$. The computational complexity can be analyzed as follows. All the three major computational steps during one iteration are considered as follows.
\begin{itemize}
  \item {\bf{Partitioning the data points}}. After initialization of the centers, this steps mainly consists of two different procedures.
  \begin{itemize}
    \item For $\lambda$-in data with data size $n_1$, this process is the same as $k$-means clustering in simply comparing the distances of data in all $c$ cluster centers. Thus, the complexity for this step is $\mathcal{O}(n_1c)$
    \item For $\lambda$-out data with data size $n_2$, the re-clustering process involves a sort operation, which the quickest complexity is $n_2\log n_2$. Under the worst case of each $\lambda$-out data being new centers, the complexity cost would be $\mathcal{O}(n_2\cdot n_2\log n_2)$. With an label assigning process, the complexity would be $\mathcal{O}(n_2^2\log n_2 + cn_2)$.
  \end{itemize}
  \item {\bf{Updating cluster centers}}. Given the partition matrix, updating the cluster centers is to find the means of the data points in the same cluster. Thus, for $c$ clusters, the computational cost complexity for this step is $\mathcal{O}(nc)$.
  \item{\bf{Agglomeration procedure}}. This procedure needs to check all the possible pairs of the clusters, thus a complexity of $\mathcal{O}(c^2)$ is needed.
\end{itemize}

Assume the clustering process needs $h$ iterations to converge, the total computational complexity of this algorithm is $\mathcal{O}(hn_2^2\log n_2+hnc+hc^2)$. While $n_2$ is usually set to be s small subset of the algorithm,
the algorithm is computational feasible. However, while we set the threshold $\lambda, \theta$ small values, leading to a larger $n_2$, then the computational cost would be heavy.

\subsection{pitman-yor spectral clustering}
Our work can also be transplanted into the spectral clustering framework. We have first shown that our objective function (Eq. (\ref{eq_2})) in \emph{pyp-means} is equivalent to
the trace optimization problem:
\begin{equation} \label{eq_3}
\max_{\{Y|Y^TY=I\}} \textrm{tr}(Y^T(K-(\lambda-\ln c\cdot\theta )I)Y)
\end{equation}
Where $K$ is the $n\times n$ kernel matrix.

Detail proof is quite similar as the one in \cite{jordanrevisiting}, we do not provide the detail here due to the duplicate.




The classical determination of the orthonormal matrices $Y$ in spectral theory states that while $Y$ selects to be the top $c$ eigenvectors, the objective function in (Eq. \ref{eq_3}) reaches its maximum for a fixed $c$ clusters. For flexible $c$ value in our problem, the objective function (Eq. \ref{eq_3}) reaches its maximum while $Y$ selected to be the matrices of eigenvectors with the non-negative eigenvalues, corresponding to the $c$-adjusted matrix $K-(\lambda-\ln c\cdot\theta )I$.

Particularly, we determine $c$, the integer number of clusters, through an adaptive measure on the $c$ value's connection to the threshold changing of the eigenvalues of the similarity matrix, i.e.:
\begin{equation}
c = \arg\max_{c\in\{1, \cdots, n\}} \{c|\lambda_c>\lambda-\ln c\cdot\theta,\lambda_{c+1}<\lambda-\ln(c+1)\cdot\theta\}
\end{equation}
Where $\{\lambda, \theta\}$ are the pre-defined parameter, $\{\lambda_i\}_{i=1}^n$ are the decreasing eigenvalues of the kernel Matrix $K$, and $\lambda_k$ denotes the $k$-th larger eigenvalue, $k=1,\cdots,n$.

After getting the relaxed cluster indicator matrices $Y$, we can cluster the rows of $Y$ as data points using $k$-means clustering, according to the standard spectral clustering method and take the corresponding result as the final clustering result.

\section{Experiments} \label{sec_7}
The experimental evaluation is conducted on three types of datasets, which are grouped into synthetic dataset, UCI benchmarking dataset \cite{Frank+Asuncion:2010} and US communities' criminal dataset \cite{Redmond2002660}. All datasets are preprocessed by normalizing each feature on each dimension into the interval [0, 1]. Furthermore, the clustering process of the algorithms is repeated for 50 times at each setting and the average value is taken as the final result. All experiments were run on a computer with Intel Xeno (R) CPU 2.53-GHz, Microsoft Windows 7 with algorithms coded in Matlab.

\subsection{Experimental Setting}
For sufficient comparison, our proposed \emph{pyp-means} are compared with three baseline algorithms: $k$-means clustering, \emph{dp-means} and Dirichlet Process variational learning (V.L.).

Parameters in these algorithms are set accordingly. In $k$-means clustering, the pre-defined cluster number is set as the true number in Synthetic data and we use the random initialization strategy as the starting partition; \emph{dp-means} and \emph{pyp-means} are using the same parameter setting, which will be described later; in V. L., we use the variational inference\cite{blei2006variational} procedure to do the learning and the related parameters are using cross validation technique to determine.



\subsection{Performance Metrics}
Validating clustering results is always a non-trivial task. Under the presence of true labels in synthetic data, we employ the accuracy to measure the effectiveness of our proposed methods, which is defined as follows:
\begin{equation}
ACC = \frac{\sum_{i=1}^n \delta(y_i, \textrm{map}(c_i))}{n}\times100
\end{equation}
Where $n$ is the data size, $y_i$ and $c_i$ denote the true label and the obtained label; $\delta(\cdot)$ is the dirac function as $
\delta(y, c)=\left\{ \begin{array}{ll}
1 & y=c;\\
0 & y\ne c.
\end{array} \right. $; map($\cdot$) is a
permutation function that maps each cluster label to a
category label, and the optimal matching can be found by
the Hungarian algorithm\cite{papadimitriou1998combinatorial}.

Besides ACC, the NMI (normalized mutual information) is also used in the synthetic data learning, i.e.,
 \begin{equation}
 NMI=\frac{\sum_{i=1}^c\sum_{j=1}^c\log(\frac{N\cdot n_{ij}}{n_in_j})}{\sqrt{\sum_{i=1}^cn_i\log{(\frac{n_i}{N})}\sum_{i=1}^cn_j\log{(\frac{n_j}{N})}}}
 \end{equation}
where $n_{ij}$ is the number of agreements between clusters $i$ and $j$, $n_i$ is the number of data points in cluster $i$, $n_j$ is the number of data points in cluster $j$, and $N$ is the total number of data points in the dataset.


\subsection{Synthetic Dataset}
To be more focused, the synthetic dataset is manually set to contain power-law behavior in our learning procedure. Here we would like to investigate multiple aspects of our method, including the clustering accuracy and NMI score performance, the relationships between threshold and discovered cluster number, running time, e.t.c..

\subsubsection{synthetic data generation}
The synthetic data is derived from the same generation algorithm as that in \cite{zhang2004fuzz}. The power-law property is reflected by specially assigning more data points to the first few clusters (to the size of about 200) while remaining others as about 30. Also, the cluster number varies from 3 to 150 to cover larger cases. Each cluster is distributed according to the 3-dimensional Gaussian Distribution $\mathcal{N}(\boldsymbol{\mu}, I)$, where $\boldsymbol{\mu}$ is one uniform distributed random centers.


\subsubsection{Practical parameter setting}
We employ the method in \cite{jordanrevisiting} to set $\lambda$'s value. We first roughly estimate the cluster number $c$ and initialize the center with the cluster mean.
From $k=1$ to $c$, we iteratively select the data point that has the largest distance (the distance is defined as the smallest distance to all the existed centers) as the new generated
center. The maximum value of distance while $k=c$ is identified as the value of $\lambda$ in our experiment. For $\theta$'s value, we experimentally set it as $\theta = \lambda/6$. Detail discussions of the $\theta$ determination will be discussed later.


\subsubsection{simulation results}
\begin{figure*}[htbp]
\centering
\includegraphics[scale=1.0, width = 1.0 \textwidth]{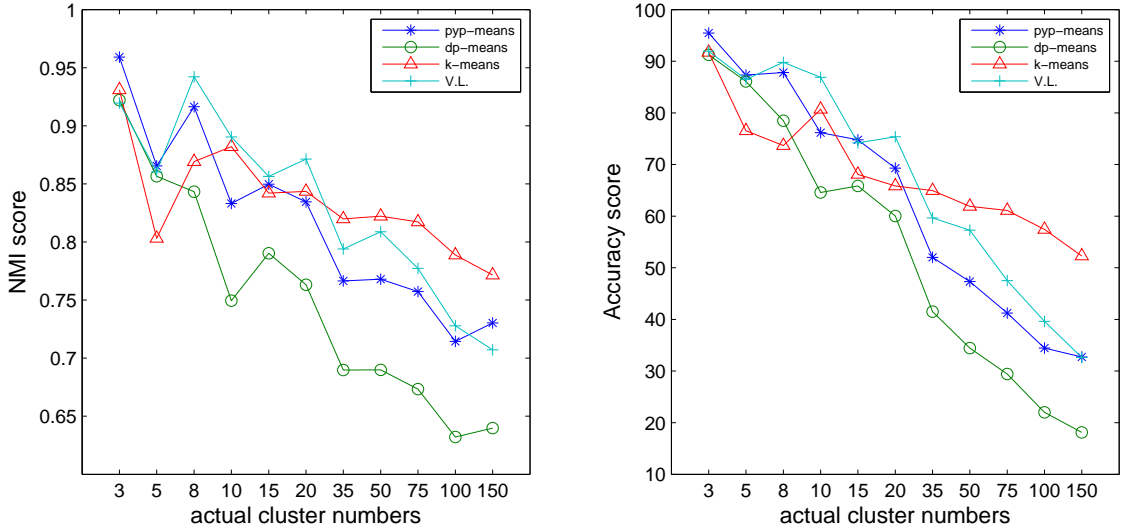}
\caption{Synthetic dataset running result}
\label{fig_3}
\end{figure*}


Figure \ref{fig_3}. shows the running result on the synthetic dataset. The experiments are running on cases with cluster number from 3 to 150 and the corresponding NMI score and accuracy score are recorded. From Figure \ref{fig_3}., it is easy to see that both \emph{dp-means} and \emph{pyp-means} get satisfied results while the cluster number is small $(c<10)$. However, when more clusters are
generated, \emph{dp-means} falls below 0.8 in NMI and 70 in accuracy while \emph{pyp-means} receives a much better performance both in NMI and accuracy. We should also note that when $c\le0$, our \emph{pyp-means} receives a better performance than $k$-means clustering in most cases, even if the later has the true cluster number.


\subsubsection{parameter learning}
\begin{figure*}[!tp]
\centering
\includegraphics[scale=1, width = 1 \textwidth]{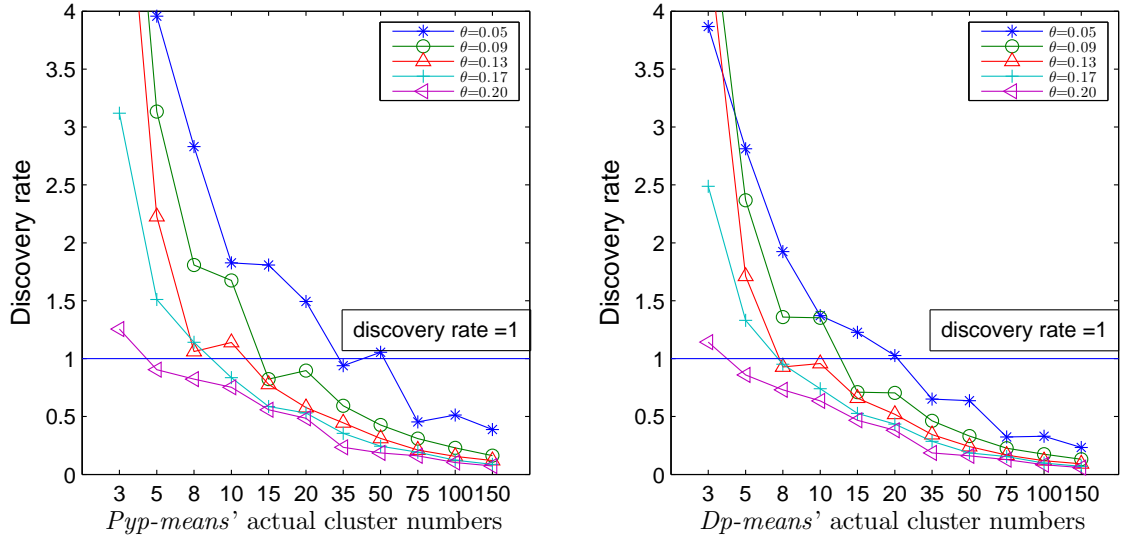}
\caption{Parameter learning result}
\label{fig_6}
\end{figure*}

In this part, the parameter $\lambda$'s value is taken from $0.05$ to $0.2$. The ``discovery rate'' ($\textrm{discovery rate} = \frac{\textrm{number of resulted clusters}}{\textrm{number of actual clusters}}$) is employed to denote the cluster number we have uncovered.
 By default, we set $\theta=\lambda/10$ in \emph{pyp-means}. The detail result shows in Figure \ref{fig_6}. From this figure, we can find that smaller threshold would in a larger discovered cluster number. This is quite reasonable as the smaller threshold would lead to smaller cluster size and then the larger cluster number. Also, our proposed \emph{pyp-means} can discovery a relative accurate cluster number while it is less than 75; however, the \emph{dp-means} can only discover perform well under the 10 cluster number case.

\subsubsection{cluster number learning}

We shows the corresponding cluster number discovered by using the parameter setting in previous in this experiment. Since $k$-means always take the true cluster number as
a prior information. We check the other three methods' discovery rate in comparison. We can see that due to the cluster number's increase, all of the discovery rate slowly
decrease. However, we can see that our \emph{pyp-means} receives a better performance than the existed $\emph{dp-means}$ while facing large cluster number situation.
\begin{figure}[!tp]
\centering
\includegraphics[scale=0.4, width = 0.4 \textwidth]{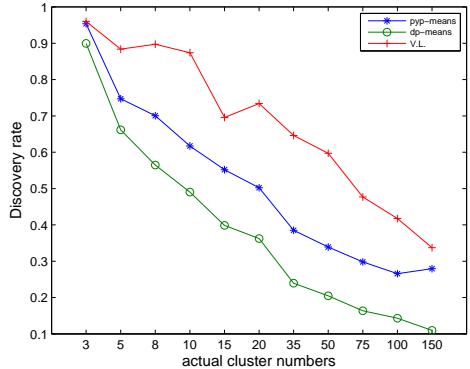}
\caption{Discovery rate}
\label{fig_4}
\end{figure}

\subsubsection{running time test}
The running time of our methods is tested to validate our complexity analysis, with methods comparable test and self-parameter comparable test. From Figure \ref{fig_7}., we can see that our \emph{pyp-means} runs approximate the same time as \emph{dp-means}. Even if in large scale case, the running time is still tolerated; while in variational learning, the running time increases in exponential.
\begin{figure}[!tp]
\centering
\includegraphics[scale=0.40, width = 0.4 \textwidth]{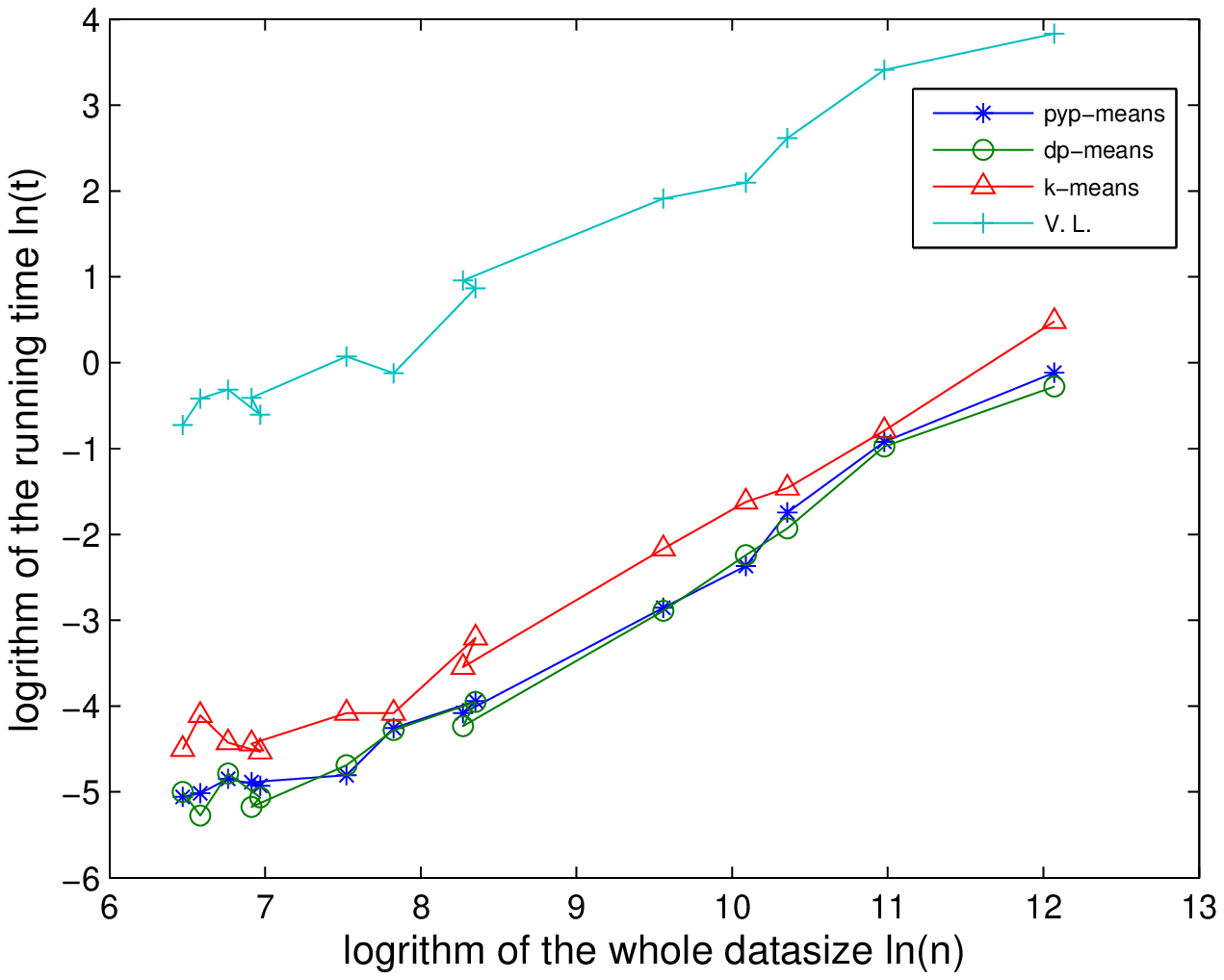}
\caption{Histogram of communities distribution}
\label{fig_7}
\end{figure}

\subsection{Real World Dataset}
Two kinds of real world dataset are used here to further validate our \emph{pyp-means}'s performance, including the UCI benchmarking dataset and the US communities Crime dataset.

\subsubsection{UCI benchmarking dataset}
In the UCI benchmarking dataset study, two types of the datasets are selected for our study:  the power-law
dataset and ``normal'' dataset with a few clusters or equal sized clusters, in contrast to the power-law dataset. We do this to show that our methods not only could receive better performance on the power law dataset but also could obtain a satisfied result on the ``normal'' dataset.


Since the power-law dataset is limited, we manually make up some by removing some data points in certain clusters in some datasets. One maximum likelihood estimation of the power-law density function parameter $\alpha$ is used to curve its detail power-law behavior, denoted as:
\begin{equation}
\hat{\alpha} = 1 + c \left[\sum_{i=1}^c \ln \frac{x_i}{x_{\min}}\right]^{-1}
\end{equation}
Here $x_i$ denotes the cluster size in our study, while $c$ represents the cluster number.

The smaller $\hat{\alpha}$, the larger power-law behavior tendency.

With the following table, the detail of UCI benchmarking datasets we are using is shown:
\begin{table}[htbp]
\caption{UCI benchmarking dataset} \label{table_2}
\centering
\begin{tabular}{c|c|cccc}
  \hline
  type &data & size & dimension & clusters & $\hat{\alpha}$  \\
    \hline
\multirow{3}{*}{normal} & wine & 178&13 & 3& 6.02\\
 & satellite & 6435 & 36 & 7 & 3.15\\
 & statlog & 2310 & 19 & 7& $\infty$\\
  \hline
\multirow{6}{*}{P. L.} & yeast &1484 &8 &10 & 1.38\\
 & vowel &349 & 10&11 & 1.33\\
 & shape &160 & 17& 9& 1.49\\
 & pendigits &7494 &16 & 10&  1.63\\
 & page-block &5473 &11 &5 & 1.49\\
 & glass & 214& 9&6 & 1.94\\
     \hline
\end{tabular}
\end{table}

We tune the parameter $\lambda$'s value experimentally to receive a better performance, and the $\theta$'s value is default as $\theta=\lambda/10$. Table \ref{table_3} is the detail outcomes of the UCI benchmarking data experiments.
\begin{table}[!tp]
\caption{UCI benchmarking dataset results} \label{table_3}
\centering
\begin{tabular}{cc|cccccccc}
  \hline
  dataset & criterion  & $pyp$-means& $dp$-means & $k$-means& V. L.  \\
  \hline
  \multirow{2}{*}{wine}&NMI  & 0.8126 & 0.7815  & 0.8349& 0.4288 \\
     &ACC& 82.04 & 80.12 & 94.94  & 62.92\\
       & C.N. & 2.90 & 2.82 & 3 & 3\\
     \hline
  \multirow{2}{*}{satellite} &NMI & 0.5953 & 0.5683  & 0.6125 & 0.3122 \\
    &ACC & 66.74 & 66.58 & 67.19  & 34.93 \\
           & C.N. & 5.96 & 5.26 & 6 & 4\\
     \hline
  \multirow{2}{*}{statlog} &NMI & 0.6537 & 0.6570 & 0.6128 & 0.4823\\
    &ACC  & 55.69 & 55.97 & 59.91 & 31.17 \\
           & C.N. & 6.82 & 6.62 & 7 & 6\\
     \hline
  \hline
  \multirow{3}{*}{yeast}&NMI  & 0.2476 & 0.1768  & 0.2711 & 0.1063 \\
     &ACC & 41.3329 & 36.2278 & 36.8706  & 33.6927\\
     & C.N. & 9 & 6.04 & 10 & 8\\
     \hline
  \multirow{2}{*}{vowel}&NMI  & 0.4479 & 0.4125  & 0.4357 & 0.3938\\
    &ACC & 28.9914 & 27.5645 & 29.5645 & 31.2321\\
     & C.N. & 11.90 & 9.04 & 11 & 11\\
          \hline
  \multirow{2}{*}{shape}&NMI & 0.7405 & 0.7204  & 0.6593& 0.4279\\
    &ACC & 64.00 & 61.36 & 63.05 & 32.50\\
         & C.N. & 15.90 & 14.90 & 9 & 9\\
     \hline
  \multirow{2}{*}{pendigits} &NMI & 0.6903 & 0.6622  & 0.6834&  0.7024 \\
     &ACC& 66.86 & 64.65 & 69.96 & 62.57 \\
     & C.N. & 10.76 & 8.88 & 10 & 10\\
     \hline
  \multirow{2}{*}{page-block} &NMI & 0.1817 & 0.1807  & 0.1484 & 0.2622 \\
    &ACC & 67.70 & 66.90 & 44.23  & 73.91\\
    & C.N. & 6.02 & 5.32 & 5 & 5\\
     \hline
  \multirow{2}{*}{glass} &NMI & 0.3875 & 0.3784 & 0.3077 & 0.2865\\
    &ACC  & 49.50 & 48.37 & 42.88 & 45.37 \\
    & C.N. & 5.94 & 5.00 & 6 & 6\\
     \hline
\end{tabular}
\end{table}
Here C.N. denotes the Cluster Number the method has produced.

From the results given, we can see that on the normal datasets, our method \emph{pyp-means} performances better or at least as good as the \emph{dp-means} and $k$-means clustering on most cases. This usually because the cluster number $c$ is usually small under this kind of dataset, leading the discount parameter $\theta$ function little in the process.

On the power-law dataset in UCI, our \emph{pyp-means} can receive better result than \emph{dp-means}. The ability to automatically learn the threshold plays a vital role in this learning. Although our methods loses at some datasets, it could still be validated valued.

\subsubsection{communities crime dataset}
We also conduct experiments on the communities crime rate dataset.

The communities crime dataset is one collection combines socio-economic data from the 1990 US Census, law enforcement data from the 1990 US LEMAS survey, and crime data from the 1995 FBI UCR.

The dataset constitute of nearly 100 attributes. The attributes varies in many aspects of the community, excluding the clearly unrelated attributes. The class label is the total number of violent crimes per $100,000$ population. In this experiment, as the crime rate is one continuous variable ranges in [0, 1], we manually discrete the values into a certain number of intervals and gets the related labels.

Figure \ref{fig_1}. depicts the histogram of the each interval's number under the case of 16 intervals.
\begin{figure}[!tp]
\centering
\includegraphics[scale=0.40, width = 0.4 \textwidth]{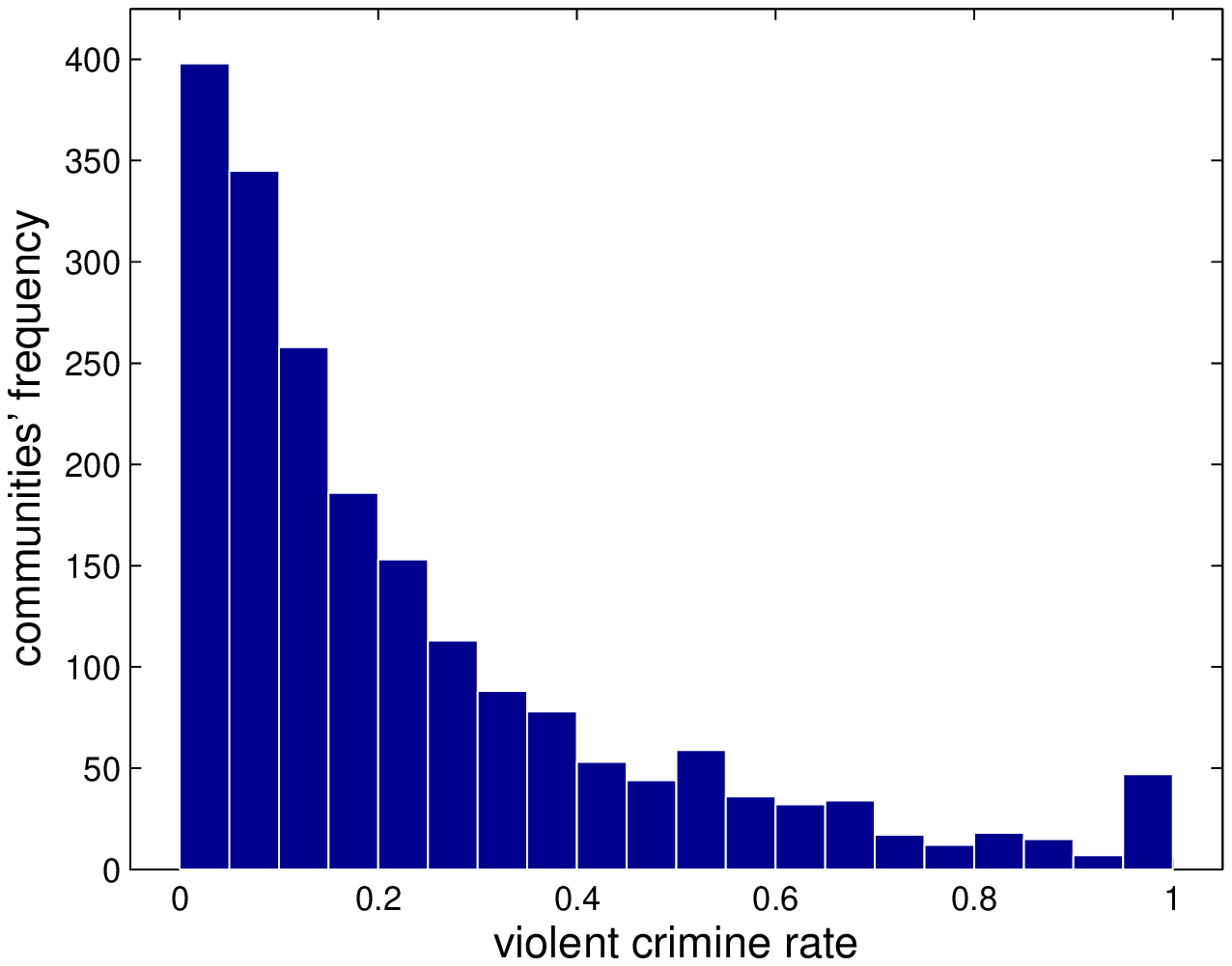}
\caption{Histogram of communities distribution}
\label{fig_1}
\end{figure}
The figure above clearly shows the data distributed according to power-law behavior.

To avoid the ``\emph{curse of dimensionality}'' problem, we apply the feature selection technique \cite{Brown:2012:CLM:2188385.2188387} and select first 10 features as the most correlated ones in advance.

The parameter are also tuned so as to better describe the true cluster labels. The tuned parameter and related results are shown in Table \ref{table_4}..
\begin{table}[htbp]
\caption{US Communities' criminal dataset results} \label{table_4}
\centering
\begin{tabular}{cc|ccccc}
  \hline
$\theta$ & criterion & \emph{pyp-means} & \emph{dp-means} & $k$-means  & V. L.  \\
  \hline
\multirow{3}{*}{6.57} & NMI & 0.2291 & 0.2075 & 0.2284  & 0.0862\\
    & ACC & 11.10 & 12.14 & 9.92 & 10.34 \\
        & C. N. & 53.8 & 34.9 & 51.0 & 13.0\\
     \hline
   \multirow{3}{*}{11.01} &  NMI & 0.1794 & 0.1680 & 0.1721  & 0.0639\\
    & ACC & 16.52 & 22.23 & 15.27 & 18.66 \\
        & C. N. & 23.5 & 10.6 & 21.0 & 11.0\\
     \hline
     \multirow{3}{*}{12.51} &  NMI & 0.1800 & 0.1715 & 0.1746  & 0.0896\\
   & ACC& 24.30 & 27.17 & 22.60 & 34.32 \\
        & C. N. & 11.7 & 7.7 & 11.0 & 9.0\\
     \hline
 \multirow{3}{*}{17.11} & NMI & 0.1737 & 0.1536 & 0.1706  & 0.1119\\
    & ACC & 40.06 & 48.13 & 35.33 & 59.46\\
        & C. N. & 5.7 & 3.6 & 6.0 & 6.0\\
     \hline
\end{tabular}
\end{table}

From the result, we can see that our \emph{pyp-means} can receive a better performance in both the NMI score and cluster number prediction.

\section{Conclusion} \label{sec_8}
One novel modified Pitman-Yor Process based method is proposed here to address the power-law data clustering problem. With the discount parameter in \emph{py-process} slightly adjusted, the power-law data is to be perfectly depicted. We also introduce one center agglomeration procedure, leading to an adaptively way in determining the number of clusters. Further, we extend our work to the spectral clustering case to address more sophisticated situations.

Some other issues are also well discussed here, including the convergence and complexity analysis, the practical issues including one reliable data clustering order. All these have greatly strengthen the solidness and reality applicability of the method.

\bibliographystyle{IEEEtran}
\bibliography{IEEEtran}

\begin{thebibliography}{10}
\providecommand{\url}[1]{#1}
\csname url@samestyle\endcsname
\providecommand{\newblock}{\relax}
\providecommand{\bibinfo}[2]{#2}
\providecommand{\BIBentrySTDinterwordspacing}{\spaceskip=0pt\relax}
\providecommand{\BIBentryALTinterwordstretchfactor}{4}
\providecommand{\BIBentryALTinterwordspacing}{\spaceskip=\fontdimen2\font plus
\BIBentryALTinterwordstretchfactor\fontdimen3\font minus
  \fontdimen4\font\relax}
\providecommand{\BIBforeignlanguage}[2]{{%
\expandafter\ifx\csname l@#1\endcsname\relax
\typeout{** WARNING: IEEEtran.bst: No hyphenation pattern has been}%
\typeout{** loaded for the language `#1'. Using the pattern for}%
\typeout{** the default language instead.}%
\else
\language=\csname l@#1\endcsname
\fi
#2}}
\providecommand{\BIBdecl}{\relax}
\BIBdecl

\bibitem{lloyd1982least}
S.~Lloyd, ``Least squares quantization in pcm,'' \emph{Information Theory, IEEE
  Transactions on}, vol.~28, no.~2, pp. 129--137, 1982.

\bibitem{hartigan1979algorithm}
J.~Hartigan and M.~Wong, ``Algorithm as 136: A k-means clustering algorithm,''
  \emph{Journal of the Royal Statistical Society. Series C (Applied
  Statistics)}, vol.~28, no.~1, pp. 100--108, 1979.

\bibitem{bezdek1980convergence}
J.~Bezdek, ``A convergence theorem for the fuzzy isodata clustering
  algorithms,'' \emph{Pattern Analysis and Machine Intelligence, IEEE
  Transactions on}, no.~1, pp. 1--8, 1980.

\bibitem{bishop2006pattern}
C.~Bishop and S.~S. en~ligne), \emph{Pattern recognition and machine
  learning}.\hskip 1em plus 0.5em minus 0.4em\relax springer New York, 2006,
  vol.~4.

\bibitem{figueiredo2002unsupervised}
M.~Figueiredo and A.~Jain, ``Unsupervised learning of finite mixture models,''
  \emph{Pattern Analysis and Machine Intelligence, IEEE Transactions on},
  vol.~24, no.~3, pp. 381--396, 2002.

\bibitem{shi2000normalized}
J.~Shi and J.~Malik, ``Normalized cuts and image segmentation,'' \emph{Pattern
  Analysis and Machine Intelligence, IEEE Transactions on}, vol.~22, no.~8, pp.
  888--905, 2000.

\bibitem{von2007tutorial}
U.~Von~Luxburg, ``A tutorial on spectral clustering,'' \emph{Statistics and
  Computing}, vol.~17, no.~4, pp. 395--416, 2007.

\bibitem{cheng1995mean}
Y.~Cheng, ``Mean shift, mode seeking, and clustering,'' \emph{Pattern Analysis
  and Machine Intelligence, IEEE Transactions on}, vol.~17, no.~8, pp.
  790--799, 1995.

\bibitem{comaniciu2002mean}
D.~Comaniciu and P.~Meer, ``Mean shift: A robust approach toward feature space
  analysis,'' \emph{Pattern Analysis and Machine Intelligence, IEEE
  Transactions on}, vol.~24, no.~5, pp. 603--619, 2002.

\bibitem{bischof1999mdl}
H.~Bischof, A.~Leonardis, and A.~Selb, ``Mdl principle for robust vector
  quantisation,'' \emph{Pattern Analysis \& Applications}, vol.~2, no.~1, pp.
  59--72, 1999.

\bibitem{fraley1998many}
C.~Fraley and A.~Raftery, ``How many clusters? which clustering method? answers
  via model-based cluster analysis,'' \emph{The computer journal}, vol.~41,
  no.~8, pp. 578--588, 1998.

\bibitem{Hamerly03learningthe}
G.~Hamerly and C.~Elkan, ``Learning the k in k-means,'' in \emph{In Neural
  Information Processing Systems}.\hskip 1em plus 0.5em minus 0.4em\relax MIT
  Press, 2003, p. 2003.

\bibitem{sugar2003finding}
C.~Sugar and G.~James, ``Finding the number of clusters in a dataset,''
  \emph{Journal of the American Statistical Association}, vol.~98, no. 463, pp.
  750--763, 2003.

\bibitem{nock2006weighting}
R.~Nock and F.~Nielsen, ``On weighting clustering,'' \emph{Pattern Analysis and
  Machine Intelligence, IEEE Transactions on}, vol.~28, no.~8, pp. 1223--1235,
  2006.

\bibitem{roweis1999unifying}
S.~Roweis and Z.~Ghahramani, ``A unifying review of linear gaussian models,''
  \emph{Neural computation}, vol.~11, no.~2, pp. 305--345, 1999.

\bibitem{jordanrevisiting}
M.~Jordan and B.~Kulis, ``Revisiting k-means: New algorithms via bayesian
  nonparametrics.''

\bibitem{kulis2011revisiting}
B.~Kulis and M.~Jordan, ``Revisiting k-means: New algorithms via bayesian
  nonparametrics,'' \emph{Arxiv preprint arXiv:1111.0352}, 2011.

\bibitem{pitman1997two}
J.~Pitman and M.~Yor, ``The two-parameter poisson-dirichlet distribution
  derived from a stable subordinator,'' \emph{The Annals of Probability},
  vol.~25, no.~2, pp. 855--900, 1997.

\bibitem{ishwaran2001gibbs}
H.~Ishwaran and L.~James, ``Gibbs sampling methods for stick-breaking priors,''
  \emph{Journal of the American Statistical Association}, vol.~96, no. 453, pp.
  161--173, 2001.

\bibitem{Clauset:2009:PDE:1655787.1655789}
\BIBentryALTinterwordspacing
A.~Clauset, C.~R. Shalizi, and M.~E.~J. Newman, ``Power-law distributions in
  empirical data,'' \emph{SIAM Rev.}, vol.~51, no.~4, pp. 661--703, Nov. 2009.
  [Online]. Available: \url{http://dx.doi.org/10.1137/070710111}
\BIBentrySTDinterwordspacing

\bibitem{selim1984k}
S.~Selim and M.~Ismail, ``K-means-type algorithms: a generalized convergence
  theorem and characterization of local optimality,'' \emph{Pattern Analysis
  and Machine Intelligence, IEEE Transactions on}, no.~1, pp. 81--87, 1984.

\bibitem{Frank+Asuncion:2010}
\BIBentryALTinterwordspacing
A.~Frank and A.~Asuncion, ``{UCI} machine learning repository,'' 2010.
  [Online]. Available: \url{http://archive.ics.uci.edu/ml}
\BIBentrySTDinterwordspacing

\bibitem{Redmond2002660}
M.~Redmond and A.~Baveja, ``A data-driven software tool for enabling
  cooperative information sharing among police departments,'' \emph{European
  Journal of Operational Research}, vol. 141, no.~3, pp. 660 -- 678, 2002.

\bibitem{blei2006variational}
D.~Blei and M.~Jordan, ``Variational inference for dirichlet process
  mixtures,'' \emph{Bayesian Analysis}, vol.~1, no.~1, pp. 121--144, 2006.

\bibitem{papadimitriou1998combinatorial}
C.~Papadimitriou and K.~Steiglitz, \emph{Combinatorial optimization: algorithms
  and complexity}.\hskip 1em plus 0.5em minus 0.4em\relax Dover Pubns, 1998.

\bibitem{zhang2004fuzz}
J.-S. Zhang and Y.-W. Leung, ``Improved possibilistic c-means clustering
  algorithms,'' \emph{Fuzzy Systems, IEEE Transactions on}, vol.~12, no.~2, pp.
  209 -- 217, april 2004.

\bibitem{Brown:2012:CLM:2188385.2188387}
G.~Brown, A.~Pocock, M.-J. Zhao, and M.~Luj\'{a}n, ``Conditional likelihood
  maximisation: A unifying framework for information theoretic feature
  selection,'' \emph{J. Mach. Learn. Res.}, vol.~13, pp. 27--66, Mar. 2012.

\end{thebibliography}

\end{document}